\documentclass[lettersize,journal]{IEEEtran}
\usepackage{amsmath,amsfonts}
\usepackage{algorithmic}
\usepackage{algorithm}
\usepackage{array}
\usepackage{textcomp}
\usepackage{stfloats}
\usepackage{url}
\usepackage{verbatim}
\usepackage{graphicx}
\usepackage{cite}
\usepackage{comment}
\usepackage{multirow}
\usepackage{microtype}
\usepackage{graphicx}
\usepackage{subfigure}
\usepackage{booktabs} 
\usepackage{color}
\usepackage{hyperref}
\usepackage{bbm}

\usepackage{amsmath}
\usepackage{amssymb}
\usepackage{mathtools}
\usepackage{amsthm}
\usepackage{multicol}
\usepackage{multirow}

\theoremstyle{plain}
\newtheorem{theorem}{Theorem}[section]
\newtheorem{proposition}[theorem]{Proposition}
\newtheorem{lemma}[theorem]{Lemma}

\theoremstyle{definition}
\newtheorem{definition}[theorem]{Definition}
\newtheorem{assumption}[theorem]{Assumption}
\theoremstyle{remark}

\usepackage{lettrine}

\begin{document}

\title{Aggregation-aware MLP: An Unsupervised Approach for Graph Message-passing}

\author{Xuanting Xie, Bingheng Li, Erlin Pan, Zhao Kang, ~\IEEEmembership{Member,~IEEE,} Wenyu Chen \thanks{Corresponding author: Zhao Kang\\X. Xie, Z. Kang, W. Chen are with the School of Computer Science and Engineering, University of Electronic Science and Technology of China, Chengdu, China; B. Li is with Michigan State University; E. Pan is with Alibaba Group (e-mail: x624361380@outlook.com; 
libinghe@msu.edu; panerlin.pel@alibaba-inc.com; \{zkang, cwy\}@uestc.edu.cn.}}


\markboth{Journal of \LaTeX\ Class Files,~Vol.~14, No.~8, August~2021}%
{Shell \MakeLowercase{\textit{et al.}}: A Sample Article Using IEEEtran.cls for IEEE Journals}


\maketitle

\begin{abstract}
Graph Neural Networks (GNNs) have become a dominant approach to learning graph representations, primarily because of their message-passing mechanisms. However, GNNs typically adopt a fixed aggregator function—such as \textbf{Mean}, \textbf{Max}, or \textbf{Sum}—without principled reasoning behind the selection. This rigidity, especially in the presence of heterophily, often leads to poor, problem-dependent performance. Although some attempts address this by designing more sophisticated aggregation functions, these methods tend to rely heavily on labeled data, which is often scarce in real-world tasks. In this work, we propose a novel unsupervised framework, ``Aggregation-aware Multilayer Perceptron'' (AMLP), which shifts the paradigm from directly crafting aggregation functions to making MLP adaptive to aggregation. Our lightweight approach consists of two key steps: First, we utilize a graph reconstruction method that facilitates high-order grouping effects, and second, we employ a single-layer network to encode varying degrees of heterophily, thereby improving the capacity and applicability of the model. Extensive experiments on node clustering and classification demonstrate the superior performance of AMLP, highlighting its potential for diverse graph learning scenarios.
\end{abstract}


\section{Introduction}
\IEEEPARstart{G}{raph} representation learning leverages Graph Neural Networks (GNNs) to enhance various tasks, showing promising results \cite{GCN,xiao2024comprehensive}. GNNs use a message-passing approach, gathering information from neighboring nodes to better differentiate between node classes by reducing intra-class variance and increasing inter-class separation \cite{zhang2024nie}. Their strong performance is often attributed to homophily, where the connected nodes belong to the same class \cite{CGC}. 
However, this assumption does not always apply, as heterophilic graphs—where connected nodes belong to different classes—are common. In these cases, GNNs that use fixed aggregation functions like \textbf{Mean}, \textbf{Max}, or \textbf{Sum} often create overly similar embeddings for different nodes, reducing their ability to distinguish between classes and leading to poor performance \cite{wang2023generalizing}. The use of fixed aggregation functions, without a principled selection, limits the flexibility of the model, underscoring the need for more adaptive methods to handle diverse graph structures effectively \cite{Raw-gnn,fang2025contributes}.

Several methods have been proposed to improve GNN performance on heterophilic graphs by adjusting the aggregation process. These approaches can be grouped into three types: (1) Label-guided aggregation: These methods adjust attention weights based on neighbors' labels, favoring nodes with the same label and reducing the influence of nodes with different labels \cite{he2022block,jin2021heterogeneous}. (2) Large receptive field aggregation: Expanding the receptive field allows the model to capture more homophilic neighborhoods, with learnable parameters to balance local and global information \cite{Global-hete, GPRGNN}. (3) Hybrid methods: These combine the first two approaches, extending non-local neighborhoods and improving message passing for both homophilic and heterophilic environments \cite{Raw-gnn,park2022deformable, yang2022graph}. 

Current aggregation-based GNNs depend heavily on labeled data, which can be expensive and time-consuming to collect. In supervised models, the aggregation process is guided by labels to align with the learned node features. However, in unsupervised learning, distinguishing between nodes becomes more difficult. Furthermore, the gap between training and test sets can lead to overfitting, where models perform well during training but fail on new data \cite{jiang2020identifying,zhang2020fairness}. These challenges point to the need for unsupervised approaches in GNNs to enhance adaptability and performance across different datasets.

 To tackle these challenges, we introduce a new approach that shifts the focus from directly learning the aggregation function to making node representations adaptable to aggregation. This allows the model to dynamically adjust to the graph structure, unlocking the full potential of multilayer perceptrons (MLPs). Unlike other unsupervised graph learning methods that use MLPs for dimension transformation, our ``Aggregation-aware MLP'' (AMLP) framework integrates the aggregation process directly into the MLPs. We also propose an aggregation-aware loss, which we empirically evaluate. The AMLP framework unifies both homophilic and heterophilic graphs through two main steps: graph reconstruction and aggregation-adaptive learning. Our main contributions are summarized as follows.
\begin{itemize}
    \item{\textit{New perspective.} To the best of our knowledge, this is the first work to shift the focus from designing the aggregation mechanism to making learned representations aggregation-adaptive. We introduce the idea that MLPs can be aggregation-aware.}
    \item{\textit{Novel algorithm.} We propose a novel, unsupervised aggregation learning framework that unifies homophilic and heterophilic graphs through graph reconstruction and aggregation-aware loss. Theoretical analysis supports its effectiveness.}
    \item{\textit{SOTA performance.} Extensive experiments on 12 homophilic and heterophilic graph datasets, comparing against 35 baselines, demonstrate the superiority of AMLP.}
\end{itemize}

\section{Related Work}
Heterophilic graphs pose a challenge for GNNs because their structure disrupts traditional aggregation mechanisms. To overcome this, spatial-based GNNs have developed various strategies to adjust edge weights during message passing. For example, FAGCN \cite{FAGCN} introduces self-gating attention, which adaptively combines low- and high-frequency signals to improve node representation. GGCN \cite{GGCN} uses structure- and feature-based edge corrections to send signed neighbor features under specific node degree constraints. RAW-GNN \cite{Raw-gnn} captures homophily through breadth-first random walks and heterophily through depth-first searches, replacing traditional neighborhoods with path-based ones and using Recurrent Neural Networks for aggregation. 

Several approaches also handle heterophilic edges by accounting for node differences or connecting distant similar nodes. GPRGNN \cite{GPRGNN} learns adaptive weights to extract information effectively, regardless of label patterns, enabling deeper networks without losing node feature discrimination. EvenNet \cite{EvenNet} introduces an even-polynomial graph filter in the spectral domain to generalize across different graph types. DeeperGCN \cite{deepergcn} proposes a generalized aggregation function and a deep residual GNN to improve learning.  Some recent methods modify edge signs or opt to skip message transmission altogether \cite{he2022block}. However, most of these approaches depend on a supervised paradigm, limiting their generalizability due to reliance on labeled data.

\begin{figure}[t]
		\centering
			\includegraphics[width=1.\linewidth]{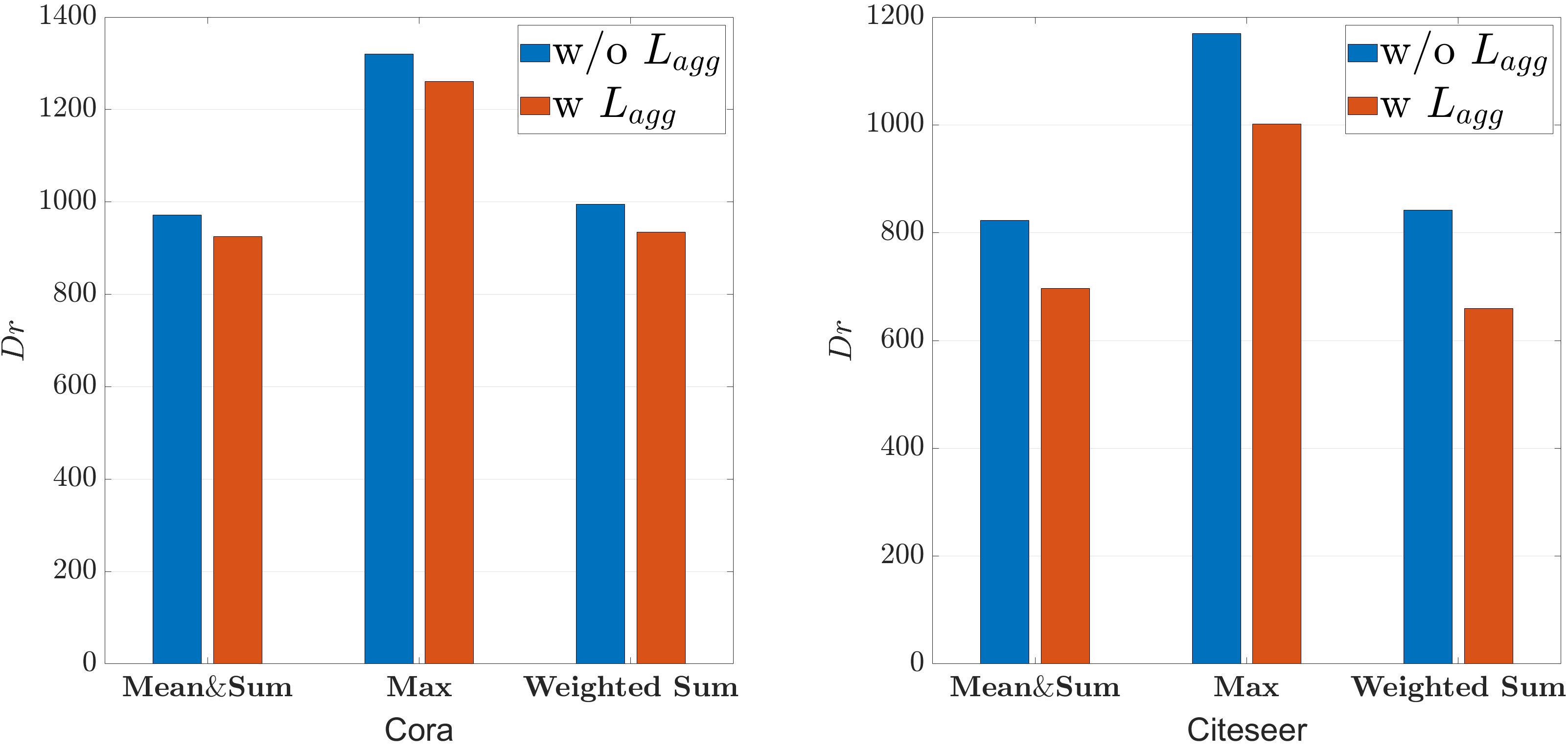}
			\includegraphics[width=1.\linewidth]{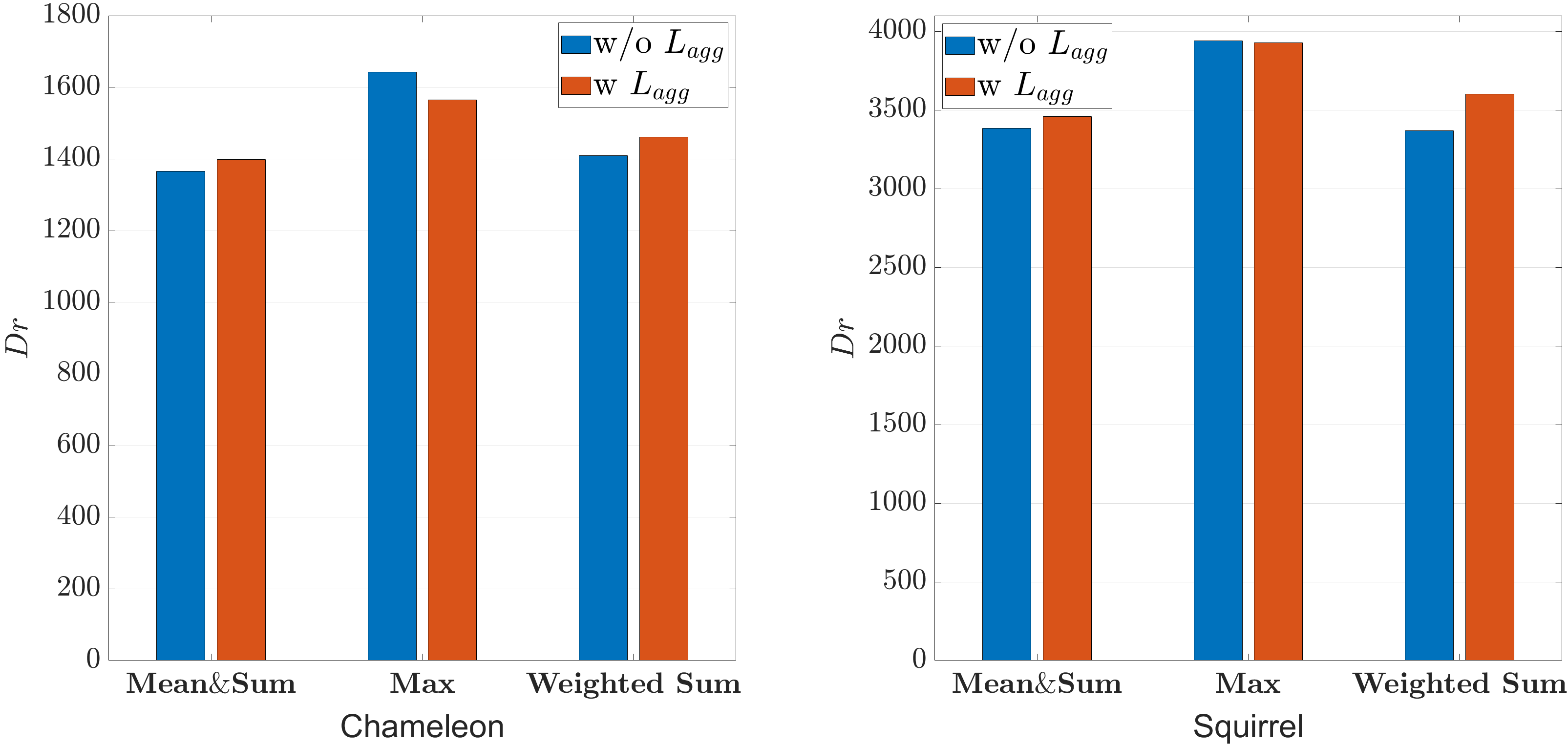}
		\caption{Results on Empirical Experiment 1.}
		\label{emp1}
\end{figure}

\begin{figure}[t]
		\centering
  \subfigure[``w/o $L_{agg}$'' on Cora.]{
			\includegraphics[width=.4\linewidth]{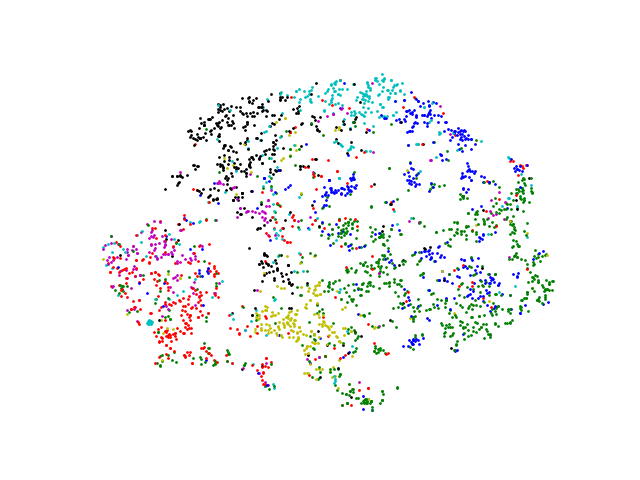}}
    \subfigure[``w $L_{agg}$'' on Cora.]{
			\includegraphics[width=.4\linewidth]{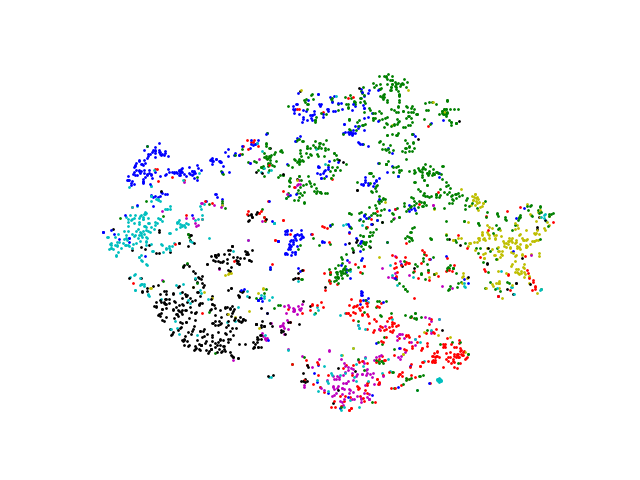}}
     \subfigure[``w/o $L_{agg}$'' on Citeseer.]{
   			\includegraphics[width=.4\linewidth]{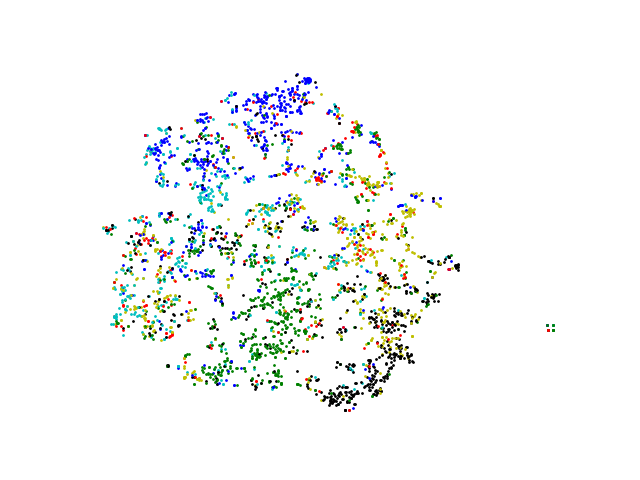}}
    \subfigure[``w $L_{agg}$'' on Citeseer.]{
			\includegraphics[width=.4\linewidth]{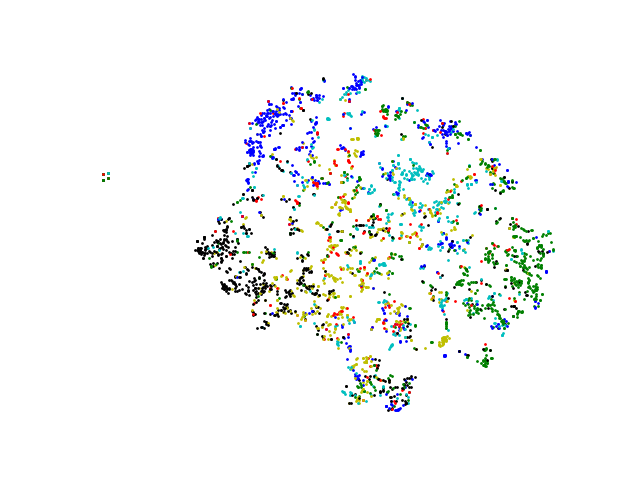}}
		\caption{Results visualization on Empirical Experiment 2.}
		\label{emp2}
\end{figure}

\section{Empirical Study}
\subsection{Preliminaries}
Assume that we have an undirected graph $\mathcal{G}=\lbrace \mathcal{V},E,X\rbrace$, where the collection of $N$ nodes is represented by $\mathcal{V}$ and the edge connecting nodes $i$ and $j$ is indicated by $e_{ij}\in E$. The feature matrix with $d$ dimensions is $X=\lbrace X_1,...,X_N\rbrace^{\top}\in \mathbb{R}^{N\times d}$. The graph structure is represented by the adjacency matrix $A \in \mathbb{R}^{N\times N}$. The diagonal degree matrix is denoted by $D_{i, i}=\sum_j A_{i, j}$. The normalized adjacency matrix is $\widetilde{A} = (D+I)^{-\frac{1}{2}}(A + I)(D+I)^{-\frac{1}{2}}$, and its related graph Laplacian is $L = I - A$.  $\hat{X}_{i} = \frac{X_i}{\left\|X_i\right\|}$. $A_{i}$ indicates the $i$-th row of $A$ and $\hat{A}_{i} = \frac{A_i}{\left\|A_i\right\|}$. $\hat{Y}$ indicates the normalized matrix of $Y$.

The message-passing-based GNNs update a node's representation $Y_v$ recursively by aggregating the representations of its neighbors. This process can be defined as:
\begin{equation}
Y_v =g\left(\left\{X_u: u \in \mathcal{N}(v)\right\} \cdot 
 W\right),
\end{equation}
where $Y_v$ is node $v$'s representation, $\mathcal{N}(v)$ is neighbor set of $v$, and $W$ is MLP. $g(\cdot)$ is an aggregation function that decides how to aggregate neighborhood representations. The choice of $g(\cdot)$ is important, which is always selected from \textbf{Mean}, \textbf{Max}, and \textbf{Sum} \cite{hamilton2017inductive}. For example, GCN \cite{GCN} adopts a \textbf{Weighted Sum} aggregation. The direct application of these aggregation functions introduces several limitations \cite{wang2023generalizing}. For \textbf{Mean} and \textbf{Sum}, the overall nonlinearity and expressiveness of GNNs are typically constrained, as deeper GNNs often encounter the over-smoothing problem due to their inherent message-passing mechanism. In the case of \textbf{Max}, it often fails to capture detailed nuances within the node representations of its neighborhood.

The unsupervised nature of our approach indicates that we cannot directly learn the proper aggregation by labels. We draw inspiration from the self-expressive model and express each node as an approximate linear combination of other data points based on adjacency \cite{MCGC}. Thus, we design the aggregation-aware loss as follows: $L_{agg} = \left\|AXW - XW\right\|_F^2$. By minimizing the discrepancy between the node representations before and after aggregation, $W$ is trained to adapt its parameters to suit different graph structures. Note that $A$ does not have the self-loop since $W$ learns nothing from it. Next, we conduct empirical experiments to see whether $W$ can capture structural disparity and be aggregation-adaptive.

\begin{figure*}[t]
\centering
\includegraphics[width=.8\textwidth]{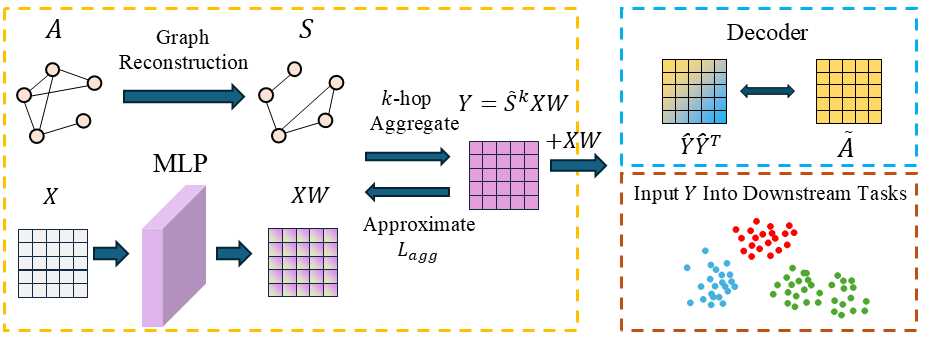} 
\caption{The framework of AMLP. It consists of two main steps: graph reconstruction to obtain $S$ and node aggregation from the $k$ hop neighbors. MLP is aware of the aggregation during training. An inner product decoder is applied to predict the original edges between nodes by the output $Y$.}
\label{all}
\end{figure*}

\subsection{Empirical Experiments}
\textbf{Empirical Experiment 1:} In this experiment, we test on a simple autoencoder network. We use a single-layer MLP and the decoding process to reconstruct $\widetilde{A}$ through minimizing $
L_{rec} = \frac{1}{N^2}\| \hat{Y}\hat{Y}^\top - \widetilde{A}\|_F^2.$

We use Dirichlet Energy ($Dr$) to measure the learned aggregation as follows:
\begin{equation}
Dr =\sum_{i, j \in \mathcal{N}} \widetilde{A}_{i j}\left\|\hat{Y}_{i}-\hat{Y}_{j}\right\|^{2}.
\end{equation}
$Dr$ measures the similarity of node representations, which serves as a quantitative measure of node aggregation. Previous work has shown that homophilic graphs prefer a small $Dr$ while heterophilic graphs do the opposite. Too large or too small $Dr$ can be considered as over-smoothing or over-sharpening. To see the effect of the aggregation-aware loss, we use the objective function $L = L_{rec}$ or $L = L_{rec} + \lambda L_{agg}$ to train the network, which is marked as ``w/o $L_{agg}$'' and ``w $L_{agg}$'', respectively. $\lambda$ is a trade-off parameter and is fixed to 0.1. We examine the widely applied aggregation functions $g(\cdot)$: \textbf{Mean}, \textbf{Max}, \textbf{Sum}, and \textbf{Weighted Sum} for both loss functions. We use two homophilic graphs: Cora and Citeseer, and two heterophilic graphs: Chameleon and Squirrel. 

The results are shown in Fig. \ref{emp1}. By adding $L_{agg}$ to train the model, $Dr$ becomes smaller on homophilic graphs in all cases, while it becomes larger on heterophilic graphs in most cases. 
The experimental results suggest that by incorporating $L_{agg}$ during training, the values of $Dr$ adjust according to the underlying graph structure. These adjustments align with previous research findings that more neighboring similarity is preferred for homophilic graphs, whereas more neighboring differences are favored for heterophilic graphs. Consequently, $L_{agg}$ helps prevent over-smoothing or over-sharpening, demonstrating the model's flexibility to different aggregation methods and graph structures.

\textbf{Empirical Experiment 2:} To see whether the aggregation-aware $W$ is beneficial for distinguishing node groups, we set $g(\cdot)$ as the widely applied \textbf{Weighted Sum} aggregation in GCN \cite{GCN} and visualize the learned node representations in \textbf{Empirical Experiment 1} through the t-SNE technique on Cora and Citeseer. 
The results are shown in Fig. \ref{emp2}. The representations obtained with $L_{agg}$ exhibit a compact and clear cluster distribution, which in turn facilitates downstream tasks.

\section{Methodology}
Based on the aforementioned insights, we propose the AMLP framework, as shown in Fig. \ref{all}. It comprises two main components. First, it reconstructs the graph structure, which is theoretically proven to have a high-order grouping effect. Second, it makes MLP aggregation-aware of global aggregation in training.

\subsection{Graph Reconstruction}
For aggregation-based GNNs, unrelated neighbors are thought to be detrimental since features of nodes from different classes would be wrongly combined, making nodes indistinguishable. Thus, the graph quality is crucial for representation. The node attribute complements graph structures with rich semantic information; however, they are not inherently
well-aligned. Therefore, we refine the original graph as follows:
\begin{equation}
\begin{aligned}
&S_{ij}=\left\{
        \begin{aligned}
			1,&     & \text { if} \quad   (\frac{X_{i}^\top \cdot X_{j}}{\|X_{i}\| \cdot \|X_{j}\|} \cdot \frac{A_{i}^\top \cdot A_{j}}{\|A_{i}\| \cdot\|A_{j}\|})^2 \geq \epsilon. \\
			0,&     &\text { otherwise} .
		\end{aligned}
		\right.\\
&i, j \in[1,2, \cdots N], \\
\label{Sequation}
\end{aligned}
\end{equation}
where $\epsilon$ is set to 0.001 or 0.05 to eliminate some noise. Consequently, the feature space and the topology space are consistent, providing reliable relation information between nodes. The corresponding refined normalized adjacency matrix is $\widetilde{S} = D^{-\frac{1}{2}}SD^{-\frac{1}{2}}$. Note that, graph reconstruction is also a common kind of fixed aggregation \cite{DGCN}. Besides, $\widetilde{S}$ doesn't have the self-loop, which is different from the traditional GNN-based methods. 

\subsection{Aggregation-aware MLP}
Most graph convolution conducts  neighborhood aggregation  as follows:
\begin{equation}
\begin{aligned}
& Y = \widetilde{S}^k XW,  \\
\end{aligned}
\end{equation}
where $W$ is MLP and $k$ is a hyper-parameter. $\widetilde{S}^k$ is the $k$-order graph filter. Traditionally, $W$ performs as a dimension transformer and $XW$ lacks structural information. $Y$ indicates the representations after aggregating information from neighbors that are $k$-hop away. 
Motivated by our empirical study, we propose to learn the aggregation-aware MLP through the aggregation-aware loss. We generalize the pre-defined $L_{agg}$ to the $k$-th layer:
\begin{equation}
L_{agg} = ||\widetilde{S}^kXW - XW||_F^2.
\end{equation}
By updating the learned representations in the $k$-th hop, the learned node representations can incorporate their structural information in a global sense. It is worth pointing out that the learning process has only one network, which leads to a parameter-efficient model.

Due to the lack of a self-loop, pushing the node embeddings close to input features encourages each node to produce an embedding that is predictable from its neighborhood context alone. A similar design can be seen in \cite{liu2024graph}, which pushes the linear combination of neighbor-aggregated features closer to the original input features. This promotes local consistency and drives the MLP to learn representations that are structurally aware—nodes whose features can be inferred from neighbors will naturally align, while outliers or structure-violating nodes will be highlighted. Besides, \cite{Global-hete} also pushes the aggregated node embeddings close to input features. It is mentioned that it can reduce noise and drive the linear representation so that nodes are close to their embeddings. Further, ResNet-based GNN \cite{li2019deepgcns} adds the original features via residual connections before aggregation to mitigate over-smoothing.

To compensate for the missing self-loop, the output of AMLP is added with raw information as follows:
\begin{equation}
Y = \widetilde{S}^k XW + XW.
\end{equation}
There are many different types of decoders, like reconstructing the attribute value, the graph structure, or both. Since our latent representation doesn't directly utilize the original structure information, we propose to use a straightforward inner product decoder to predict the edges between nodes:
\begin{equation}
L_{rec}=\frac{1}{N^2}\|\hat{Y}\hat{Y}^\top - \widetilde{A} \|_F^2.
\end{equation}
Finally, the total loss is formulated as:
\begin{equation}
L= L_{agg} + \lambda L_{rec},
\end{equation}
where $\lambda$ is the trade-off parameter. 
Our aggregation-ware loss can be integrated into any existing message-passing GNNs. 

\subsection{Theoretical Analysis}
We provide theoretical justification for our method: (1) the filtered features effectively encode both node attributes and topological structures, enabling a high-order grouping effect that extends traditional grouping effects by incorporating higher-order information; (2) the aggregation-aware loss function achieves a high-pass filtering effect, enhancing the model's ability to capture fine-grained structural details.


\begin{definition}
\textbf{(Grouping effect \cite{GEKDD})} There are two similar nodes $i$ and $j$ in terms of local topology and node features, i.e., $V_i \to V_j \Leftrightarrow \left( \| \hat{A}_i - \hat{A}_j \|_2 \to 0 \right) \wedge \left( \| \hat{X}_i - \hat{X}_j \|_2 \to 0 \right)$, the matrix $G$ is said to have a grouping effect if $V_i \to V_j \Rightarrow |G_{ip} - G_{jp}| \to 0, \forall 1 \leq p \leq N$.
\end{definition}

In graph learning, the grouping effect naturally arises as similar nodes aggregate into closer embeddings, capturing shared characteristics and improving representation learning. This phenomenon has been widely recognized as a fundamental principle for enhancing model performance \cite{CDC}. While traditional grouping effects primarily emphasize local similarities, real-world graphs often exhibit complex higher-order structures. To address this, we extend the concept by incorporating high-order proximities, enabling better alignment of embeddings beyond immediate neighbors, as described below.

\begin{theorem}
\textbf{(High-order grouping effect)} Let the distance between filtered nodes \(i\) and \(j\) be \( \| Y_i - Y_j \| \).
With probability at least $1-\delta$, the following upper bound holds:
\[
\| Y_i - Y_j \| \leq O \left( \sqrt{ \ln\left(\frac{2}{\delta} \cdot \left( 
\mathcal{M} + \mathcal{N} \right) \right) } \right)
\]
\[
\mathcal{M} = \|\hat{X}_i - \hat{X}_j\|^2 + \|\hat{A}_i - \hat{A}_j\|^2
\]
\[
\mathcal{N}=\sum_{m \in {[N] \setminus \{i, j\} }}\left\|\hat{X}_i^\top \hat{X}_m - \hat{X}_j^\top \hat{X}_m \right\| + \left\|\hat{A}_i^\top \hat{A}_m - \hat{A}_j^\top \hat{A}_m \right\|
\]

\label{the1}
 The filtered features \(Y_i\) and \(Y_j\) preserve both the attribute-based similarity, the topological similarity, and the high-order similarity between nodes \(i\) and \(j\).
\end{theorem}

\begin{proof}
    \begin{lemma}[Hoeffding's Inequality]\label{Hoeffding's Inequality}
Let \( X_1, X_2, \dots, X_n \) be independent random variables such that \( X_i \in [a_i, b_i] \) almost surely for each \( i \). Define the sum \( S_n = \sum_{i=1}^n X_i \). Then, for any \( t > 0 \), the following inequality holds:
\[
P\left( S_n - \mathbb{E}[S_n] \geq t \right) \leq \exp\left( -\frac{2t^2}{\sum_{i=1}^n (b_i - a_i)^2} \right),
\]
and similarly,
\[
P\left( S_n - \mathbb{E}[S_n] \leq -t \right) \leq \exp\left( -\frac{2t^2}{\sum_{i=1}^n (b_i - a_i)^2} \right).
\]
\end{lemma}

\begin{assumption}\label{assumption1}
Each element of the matrix \(W \in \mathbb{R}^{d \times c}\) is sampled uniformly from the interval \(\left[-\frac{1}{\sqrt{c}}, \frac{1}{\sqrt{c}}\right]\). For any \(r \in \mathbb{R}^c\), it holds that \(\mathbb{E}[\|r\cdot W \|_p] = 0\). Furthermore, by Hoeffding's inequality:
\[
P\left( \|r W \|_p \geq t \right) \leq 2 \exp\left( -\frac{d t^2}{2 \|r\|_p^2} \right),
\]
\end{assumption}

\begin{lemma}[Lipschitz Continuity of the Sigmoid Function]\label{Lipschitz}
The Sigmoid function \( f(x) = \frac{1}{1 + e^{-\mathbf{k}(x - \epsilon)}} \) is Lipschitz continuous with a Lipschitz constant \( \mathbf{L} = \frac{\mathbf{k}}{4} \), i.e. 
\[
|f(x_1) - f(x_2)| \leq \frac{\mathbf{k}}{4} |x_1 - x_2|, \quad \forall x_1, x_2 \in \mathbb{R}.
\]
\end{lemma}

Eq. (\ref{Sequation}) refine the original graph and defines \(S_{ij}\) as a binary variable, which has a hard threshold : 
\[
S_{ij} =
\begin{cases}
1, & \text{if } \left( \frac{\mathbf{X}_i^\top \mathbf{X}_j}{\|\mathbf{X}_i\| \|\mathbf{X}_j\|} \cdot \frac{\mathbf{A}_i^\top \mathbf{A}_j}{\|\mathbf{A}_i\| \|\mathbf{A}_j\|} \right)^2 \geq \epsilon, \\
0, & \text{otherwise.}
\end{cases}
\]
In the proof, we can replace the binary decision metric with a continuous function \(f\), such that:
\[
S_{ij} \approx f\left(\left( \frac{\mathbf{X}_i^\top \mathbf{X}_j}{\|\mathbf{X}_i\| \|\mathbf{X}_j\|} \cdot \frac{\mathbf{A}_i^\top \mathbf{A}_j}{\|\mathbf{A}_i\| \|\mathbf{A}_j\|} \right)^2\right).
\]
The Sigmoid function is a smooth, continuous, and monotonic function, defined as:
\[
f(x) = \frac{1}{1 + e^{-\mathbf{k}(x - \epsilon)}}.
\]
Applying this approximation to the original formula:
\[
S_{ij} \approx \frac{1}{1 + e^{-\mathbf{k}\left(\left( \frac{\mathbf{X}_i^\top \mathbf{X}_j}{\|\mathbf{X}_i\| \|\mathbf{X}_j\|} \cdot \frac{\mathbf{A}_i^\top \mathbf{A}_j}{\|\mathbf{A}_i\| \|\mathbf{A}_j\|} \right)^2 - \epsilon\right)}}.
\]
Here \(\mathbf{k}\) controls the smoothness. As \(\mathbf{k} \to \infty\), the Sigmoid function approaches a hard threshold. \(\epsilon\) is the original threshold. In the following proof, we use the general and smoothed  $S_{ij}$ for analysis instead.



We first compute the distance between the filtered features \(Y^k_i\) and \(Y^k_j\) of nodes \(i\) and \(j\) after \(k\) layers of filtering, denoted as \(\| Y^k_i - Y^k_j \|\). This can be expressed as:
\[
\begin{aligned}
 \| Y^k_i - Y^k_j \| &= \left\| (S^k X W)_i - (S^k XW)_j + (XW)_i - (XW)_j\right\| \\
 & = \left\| [(S^k X)_i - (S^k X)_j + X_i - X_j] W\right\| \\
\end{aligned}
\]
Note \(R_{ij}\) = \((S^k X)_i - (S^k X)_j + X_i - X_j\) . Using Lemma \ref{Hoeffding's Inequality} and Assumption \ref{assumption1}, and we can get
\[
\| Y^k_i - Y^k_j \| = \left\| R_{i j} \cdot W\right\|  \leq \left\| \sqrt{-\frac{2 \|R_{ij}\|}{d} \ln\left(\frac{\delta}{2}\right)} \right\| 
 \]
By factoring out \((S_i - S_j)\), we obtain:
\[
R_{ij} = \left\| (S_i - S_j) \cdot S^{k-1}X + X_i - X_j\right\|
\]
Using the triangle inequality, we can upper-bound the distance as follows:

\[
R_{ij} \leq \|S_i - S_j\| \cdot \left\|S^{k-1}X \right\| + \left\|X_i - X_j\right\|\]
Since \(S_i - S_j\) reflects the difference in topological structure between nodes \(i\) and \(j\), we can have :
\begin{align*}
\|S_i - S_j\| &= \|\sum_m S_{im} - S_{jm}\| \\=  \| S_{ij} - S_{ii} \| + &\| S_{ij} - S_{jj} \| + \sum_{m \in {[N] \setminus \{i, j\} }} \| S_{im} - S_{jm} \|  
\end{align*}


Now, we analyze the two terms \(\| S_{ij} - S_{ii} \| + \| S_{ji} - S_{jj} \|\) and \(\sum\limits_{m \in {[N] \setminus \{i, j\} }} \| S_{im} - S_{jm}\|\)separately.

\textbf{1. Handling the Term \(\| S_{ij} - S_{ii} \| + \| S_{ji} - S_{jj} \|\)}

Using Lemma \ref{Lipschitz}, we can have
\[
\begin{aligned}
&\| S_{ij} - S_{ii} \| \\
&\leq \frac{1}{4}\left\|\frac{X_i^\top X_j}{\|X_i\| \|X_j\|} \cdot \frac{A_i^\top A_j}{\|A_i\| \|A_j\|}   - \frac{X_i^\top X_i}{\|X_i\| \|X_i\|} \cdot
\frac{A_i^\top A_i}{\|A_i\| \|A_i\|}\right\| \\
&=  \left\|\frac{X_i^\top X_j}{\|X_i\| \|X_j\|} \cdot \frac{A_i^\top A_j}{\|A_i\| \|A_j\|}   - 1\right\|
\end{aligned}
\]
This term can be expanded as:
\[
 \left\| \frac{X_i^\top X_j}{\|X_i\| \|X_j\|} \cdot \frac{A_i^\top A_j}{\|A_i\| \|A_j\|} 
- \frac{A_i^\top A_j}{\|A_i\| \|A_j\|} 
+  \frac{A_i^\top A_j}{\|A_i\| \|A_j\|} - 1 \right\|
\]
or 
\[ \left\| \frac{X_i^\top X_j}{\|X_i\| \|X_j\|} \cdot \frac{A_i^\top A_j}{\|A_i\| \|A_j\|} - \frac{X_i^\top X_j}{\|X_i\| \|X_j\|} +  \frac{X_i^\top X_j}{\|X_i\| \|X_j\|} - 1 \right\|
\]

Using the Cauchy-Schwarz inequality, we can upper-bound the above expression as:
\[
\begin{aligned}
&\| S_{ij} - S_{ii} \| + \| S_{ji} - S_{jj} \| \\
&\leq \frac{1}{2}\min \left\{ 
\left\| \frac{X_i^\top X_j}{\|X_i\| \|X_j\|} - 1 \right\| \cdot \left\| \frac{A_i^\top A_j}{\|A_i\| \|A_j\|} \right\| + \left\| \frac{A_i^\top A_j}{\|A_i\| \|A_j\|} - 1 \right\|, \right. \\
 &\quad \left. \left\| \frac{X_i^\top X_j}{\|X_i\| \|X_j\|} - 1 \right\| 
+ \left\| \frac{A_i^\top A_j}{\|A_i\| \|A_j\|} - 1 \right\| \cdot \left\| \frac{X_i^\top X_j}{\|X_i\| \|X_j\|} \right\| \right\}.
\end{aligned}
\]

i.e.
\[
\begin{aligned}
&\| S_{ij} - S_{ii} \| + \| S_{ji} - S_{jj} \| \\
&\leq \min \left\{ \|\hat{A}_i - \hat{A}_j\|^2 + \|\hat{X}_i - \hat{X}_j\|^2 \cdot 
\left\| \frac{A_i^\top A_j}{\|A_i\| \|A_j\|} \right\|, \right. \\
 &\quad \left. \|\hat{X}_i - \hat{X}_j\|^2 + \|\hat{A}_i - \hat{A}_j\|^2 \cdot \frac{X_i^\top X_j}{\|X_i\| \|X_j\|} \right\}.
\end{aligned}
\]
where $\hat{X} = \frac{X}{||X||}$ and $\hat{A} = \frac{A}{||A||}$.
Thus, we arrive at:
\[
\| S_{ij} - S_{ii} \| + \| S_{ji} - S_{jj} \|  \leq O \left( \|\hat{X}_i - \hat{X}_j\|^2 + \|\hat{A}_i - \hat{A}_j\|^2 \right)
\]


\textbf{2. Handling the Neighborhood Difference Term}

For the second term involving the neighborhoods of nodes \(i\) and \(j\), we have:
\[
\begin{aligned}
&\sum_{m \in {[N] \setminus \{i, j\} }} \| S_{im} - S_{jm} \| \\
&\leq \sum_{m \in {[N] \setminus \{i, j\} }} \frac{1}{4} \| \frac{X_i^\top X_m}{\|X_i\| \|X_m\|} \cdot \frac{A_i^\top A_m}{\|A_i\| \|A_m\|}\\
&\hspace{1cm}- \frac{X_j^\top X_m}{\|X_j\| \|X_m\|} \cdot \frac{A_j^\top A_m}{\|A_j\| \|A_m\|} \| \\
&=\sum_{m \in {[N] \setminus \{i, j\} }}\|\hat{X}_i^\top\hat{X}_m \cdot \hat{A}_i^\top\hat{A}_m -
\hat{X}_i^\top\hat{X}_m \cdot \hat{A}_j^\top\hat{A}_m \\
&\hspace{1cm}+ \hat{X}_i^\top\hat{X}_m \cdot \hat{A}_j^\top\hat{A}_m - \hat{X}_j^\top\hat{X}_m \cdot \hat{A}_j^\top\hat{A}_m \| \\
&\leq \sum_{m \in {[N] \setminus \{i, j\} }}\min\{\left\|\hat{X}_i^\top\hat{X}_m\right\| \cdot \left\|\hat{A}_i^\top\hat{A}_m - \hat{A}_j^\top\hat{A}_m \right\| \\
&+ \left\|\hat{A}_j^\top\hat{A}_m\right\|\cdot\|\hat{X}_i^\top\hat{X}_m - \hat{X}_j^\top\hat{X}_m ||, \left\|\hat{X}_j^\top\hat{X}_m\right\| \cdot \left\|\hat{A}_i^\top\hat{A}_m - \hat{A}_j^\top\hat{A}_m 
\right\|\\
&+ \left\|\hat{A}_i^\top\hat{A}_m\right\|\cdot\|\hat{X}_i^\top\hat{X}_m - \hat{X}_j^\top\hat{X}_m ||\}\\
&\leq O( \sum_{m \in {[N] \setminus \{i, j\} }}\left\|\hat{A}_i^\top\hat{A}_m - \hat{A}_j^\top\hat{A}_m 
\right\| + \|\hat{X}_i^\top\hat{X}_m - \hat{X}_j^\top\hat{X}_m \|)
\end{aligned}
\]
The difference in neighborhoods directly contributes to the overall distance.

\textbf{3. Combining the Results}

By combining all the terms, we obtain the final conclusion:

\[
\| Y_i - Y_j \| \leq O \left( \sqrt{ \ln\left(\frac{2}{\delta} \cdot \left( 
\mathcal{M} + \mathcal{N} \right) \right) } \right)
\]
\[
\mathcal{M} = \|\hat{X}_i - \hat{X}_j\|^2 + \|\hat{A}_i - \hat{A}_j\|^2
\]
\[
\mathcal{N}= \sum_{m \in {[N] \setminus \{i, j\} }}\left\|\hat{X}_i^\top \hat{X}_m - \hat{X}_j^\top \hat{X}_m \right\| + \left\|\hat{A}_i^\top \hat{A}_m - \hat{A}_j^\top \hat{A}_m \right\|.
\] \\
\(\boxed{\text{Q.E.D.}}\)
\end{proof}

Thus, if $\left( \| \hat{A}_i - \hat{A}_j \| \to 0 \right) \wedge \left( \| \hat{X}_i - \hat{X}_j \| \to 0 \right) \wedge \left(\left\|\hat{X}_i^\top \hat{X}_m - \hat{X}_j^\top \hat{X}_m \right\| \to 0\right) \wedge \left(\left\|\hat{A}_i^\top \hat{A}_m - \hat{A}_j^\top \hat{A}_m\right\| \to 0 \right)$,  then \(Y_i\) $\to$ \(Y_j\), which captures the high-order grouping effect. This implies that the filtered features \( Y \) preserve both high-order topological and attribute-based similarity between nodes.

\textbf{Comment.} High-order grouping effect expresses the following idea: The similarity between nodes is not just local, but should also include whether their interaction patterns with all other nodes are consistent. Take an example with four nodes: $A = [[0,1,0,1], [1,0,1,0], [0,1,0,1], [1,0,1,0]], X = [[1,1,-1],[0,1,0],[1,1,1], [1,0,0]]$. $\|X_1 - X_3\|>0$, which indicates that grouping effect fails to identify nodes 1 and 3 as one cluster. But they have a similar high-order relationship, i.e. $\|X_1X^{\top}_2-X_3X_2^{\top}\|+\|X_1X_4^{\top}-X_3X_4^{\top}\| = 0$ and $\|A_1A^{\top}_2-A_3A_2^{\top}\|+\|A_1A_4^{\top}-A_3A_4^{\top}\| = 0$, which indicates they are possibly from the same cluster with the high-order grouping effect.

\begin{proposition}
The loss $L_{agg}$ exhibits a high-pass filtering effect in the spectral domain, which allows it to effectively diffuse or push apart the neighbors of adjacent $k$-hop nodes.
\label{the2}
\end{proposition}

\begin{proof}
\begin{equation}
\begin{aligned}
L_{agg} &=||\widetilde{S}^kXW - XW||_F^2\\ 
&= Tr(|(XW)^\top (\widetilde{S}^k - I)^2 XW|)\\
\min L_{agg} &= \min Tr(|(XW)^\top (\widetilde{S}^k - I)^2 XW|)\\ &= \max \sum_{u j} [(\widetilde{S}^k - I)^2-I]_{u j } ||(XW)_u-(XW)_j||\\
\end{aligned}
\end{equation}
Thus, $L_{agg}$ exhibits a high-pass filtering effect in the spectral domain, which allows it to effectively diffuse or push apart the neighbors of the adjacent $k$-hop nodes. \\
\(\boxed{\text{Q.E.D.}}\)
\end{proof}

Note that our aggregation typically acts as a low-pass filter, smoothing node representations by averaging information from neighboring nodes \cite{CGC}. This interplay between the low-pass effect of message passing and the high-pass effect of loss allows the model to capture both the smoothness of local node neighborhoods and the sharp structural or feature distinctions within the graph.
 

 \begin{table*}[t]
		\centering
	\caption{Statistics information of datasets.}
		\label{tab::datasets}%
    \resizebox{.6\textwidth}{!}{
    \begin{tabular}{ccccccc}
    \toprule
    \multicolumn{2}{c}{Graph Datasets} & Nodes & Dims. & Edges & Clusters & Homophily Ratio \\
    \midrule
    {Heterophilic Graphs} & Cornell & 183   & 1703  & 298   & 5     & 0.1220  \\
          & Texas & 183   & 1703  & 325   & 5      & 0.0614  \\
          & Washington  & 230   & 1703  & 786   & 5    & 0.1434 \\
          & Chameleon & 2277 & 2325  & 31371 & 5    & 0.2299 \\
          & Squirrel & 5201  & 2089  & 217073 & 5    & 0.2234  \\
          & Roman-empire & 22662  & 300 & 32927 & 18    & 0.0469  \\
    \midrule
    {Homophilic Graphs}& Cora  & 2708  & 1433  & 5429  & 7    & 0.8137  \\
          & Citeseer & 3327  & 3703  & 4732  & 6     & 0.7392  \\
          & Pubmed & 19717 &500 &44327 &3 &0.8024 \\
          & UAT   & 1190  & 239  & 13599 & 4     & 0.6978 \\
          & AMAP  & 7650  & 745   & 119081 & 8    & 0.8272 \\
          & Ogbn-arXiv &169,343 &128 &1,166,243 &40 &0.6778 \\
    \bottomrule
    \end{tabular}}%
	\end{table*}

\section{Experiments On Clustering}
We experimentally evaluate the performance of AMLP in clustering tasks using real-world benchmark datasets. The clustering outcomes are obtained by applying K-means on $\hat{Y}$.

\subsection{Datasets}
 To fully evaluate the proposed AMLP, we select both homophilic and heterophilic datasets that are commonly used. For heterophilic graphs, six datasets are selected: Chameleon and Squirrel are Wikipedia-sourced page-page networks on specific topics \cite{rozemberczki2021multi}, Cornell, Texas, Washington\footnote{http://www.cs.cmu.edu/afs/cs.cmu.edu/project/theo-11/www/wwkb/}, and Roman-Empire \cite{Roman}. Chameleon and Squirrel are Wikipedia-sourced page-page networks on specific topics; Cornell, Texas, and Washington are the webgraphs from several university computer science departments; Roman-Empire is derived from one of the longest English Wikipedia articles. For homophilic graphs, six datasets are used: Cora, CiteSeer, and Pubmed are the citation networks of scientific publications \cite{GCN}, Amazon Photo (AMAP) is extracted from Amazon co-purchase graph \cite{DCRN}, and USA Air-Traffic (UAT) is collected from the Bureau of Transportation Statistics \cite{EAT}. We also add a large dataset Ogbn-arXiv, which is a paper citation network of ARXIV papers \cite{hu2020open}. The calculation of the homophily ratio follows \cite{Geom-GCN}, with larger values indicating higher homophily. Dataset statistics are in Table \ref{tab::datasets}. According to \cite{luan2024heterophily}, Cornell, Wisconsin, Texas, and Roman-Empire are particularly challenging for GNN models.

\subsection{Clustering Baselines}
The unsupervised nature of existing clustering methods limits the ability to learn adaptive aggregation schemes, as they typically rely on fixed weighted sum aggregation functions. To highlight the effectiveness of AMLP, we compare it with 26 baseline methods from six categories: 1) traditional GNN-based methods, such as DAEGC \cite{DAEGC}, MSGA \cite{MSGA}, SSGC \cite{S2GC}, GMM \cite{CDRS}, RWR \cite{RWR}, ARVGA \cite{ARVGA}, DMGNC \cite{DMGNC}, DyFSS \cite{DyFSS}; 2) contrastive learning-based methods, such as MVGRL \cite{MVGRL}, SDCN \cite{SDCN}, DFCN \cite{DFCN}, DCRN \cite{DCRN}, SCGC \cite{SCGC}, and CCGC \cite{CCGC}, designing new data augmentations to enhance the learned representations; 3) the advanced clustering approach AGE \cite{AGE}, which uses adaptive encoders and Laplacian smoothing filters to produce representations that are suitable for clustering; 4) shallow methods that utilize a fixed graph filter to smooth or sharpen the raw features, including MCGC \cite{MCGC}, FGC \cite{FGC} and RGSL \cite{RGSL}. Specifically, RGSL designs a novel $\alpha$-norm to tackle heterophily; (5) recent methods that unify homophily and heterophily, including SELENE \cite{SELENE}, CGC \cite{CGC}, and DGCN \cite{DGCN}. SELENE distinguishes between r-ego networks based on node characteristics and structural information independently using a dual-channel feature embedding method. CGC uses a hyper-parameter to combine both high- and low-pass filters to obtain better representations before graph structure learning for clustering. DGCN reconstructs homophilic and heterophilic graphs separately. Effective graph reconstruction is always expensive, as DGCN requires $O(N^4)$ complexity. (6) Scalable methods. BGRL \cite{BGRL}, ProGCL \cite{progcl}, S$^3$GC \cite{S3GC}, and Dink-net \cite{Dink-net}. They are designed to handle large-scale graphs.

\begin{table*}[htbp]
		\centering
  \caption{Clustering results on heterophilic graphs. The best results are marked in \textcolor[rgb]{ 1,  0,  0}{\textbf{red}}. The runner-up is marked in \textbf{bold}.}
		\label{Reheter}%
    \resizebox{.75\textwidth}{!}{
    \begin{tabular}{ccccccccccccc}
    \toprule
    \multirow{2}[4]{*}{Methods} & \multicolumn{2}{c}{Cornell} & \multicolumn{2}{c}{Texas} & \multicolumn{2}{c}{Washington} & \multicolumn{2}{c}{Chameleon} & \multicolumn{2}{c}{Squirrel}  & \multicolumn{2}{c}{Roman-empire}\\
\cmidrule{2-13}          & ACC   & NMI   & ACC   & NMI   & ACC   & NMI   & ACC   & NMI   & ACC   & NMI   & ACC   & NMI \\
    \midrule
    DAEGC & 42.56  & 12.37 &45.99 &11.25 & 46.96  &17.03   & 32.06  & 6.45   & 25.55  & 2.36 &21.23 &12.67\\
    MSGA  & 50.77  & 14.05 &57.22 &12.13 & 49.38   & 6.38   & 27.98   &6.21   & 27.42  & 4.31 &19.31 &12.25\\
    FGC   & 44.10  & 8.60   &53.48 &5.16 & 57.39  & 21.38   & 36.50 &11.25   & 25.11  & 1.32 &14.46 &4.86 \\
    GMM  & 58.86 & 9.26   &58.29 &13.06  & 60.86  & 20.56 & 34.91   & 7.89   &29.76   &5.15 & 21.90   & 13.57 \\
    RWR  & 58.29 & 11.35   &57.22 &13.87 & 63.91  & 23.13 & 33.27   & 8.03   & 29.96   & 5.69 & 22.73   & 14.74\\
    ARVGA  & 56.23  & 9.35   &59.89 &16.37 & 60.87  & 16.19 & 37.33   & 9.77   & 25.32   & 2.88 & 22.89   & 15.25\\
    DCRN &51.32 &9.05 &57.74 &19.86 &55.65 &14.15 &34.52 &9.11 &30.69 &6.84 &32.57 &29.50 \\
    SELENE &57.96&17.32 &65.23&25.40 &-&- &\textbf{38.97}&\textbf{20.63} &-&- &-&-\\
    CGC &44.62 &14.11 &61.50 &21.48 &63.20&25.94 &36.43&11.59 &27.23 &2.98 &30.16 &27.25\\
    DGCN & \textbf{62.29}  &\textbf{29.93} & \textbf{72.68} &33.67 & \textbf{69.13} &28.22 &36.14  &11.23 &\textbf{31.34} &7.24 & 33.42   &\textbf{31.44} \\
    RGSL &57.44 &28.95 &72.19 &\textbf{37.86} &66.09 &\textbf{29.79} &38.52 &12.79 &30.74 &\textbf{8.74} &\textbf{34.57} &31.23 \\

    AMLP &\textcolor[rgb]{ 1,  0,  0}{\textbf{66.67}}&\textcolor[rgb]{ 1,  0,  0}{\textbf{33.65}} &\textcolor[rgb]{ 1,  0,  0}{\textbf{74.32}}&\textcolor[rgb]{ 1,  0,  0}{\textbf{42.51}} &\textcolor[rgb]{ 1,  0,  0}{\textbf{74.35}} &\textcolor[rgb]{ 1,  0,  0}{\textbf{46.74}} &\textcolor[rgb]{ 1,  0,  0}{\textbf{43.21}}&\textcolor[rgb]{ 1,  0,  0}{\textbf{23.17}}&\textcolor[rgb]{ 1,  0,  0}{\textbf{33.40}}&\textcolor[rgb]{ 1,  0,  0}{\textbf{8.87}} &\textcolor[rgb]{ 1,  0,  0}{\textbf{36.04}}&\textcolor[rgb]{ 1,  0,  0}{\textbf{35.07}}\\

    \bottomrule
    \end{tabular}}%
	\end{table*}%

 \begin{table*}[!htbp]
		\centering
  \caption{Clustering results on homophilic graphs.}
		\label{Rehomo}%
    \resizebox{.7\textwidth}{!}{
        \begin{tabular}{cccccccccccc}
    \toprule
    \multirow{2}[4]{*}{Methods} & \multicolumn{2}{c}{Cora} & \multicolumn{2}{c}{Citeseer}  & \multicolumn{2}{c}{Pubmed} & \multicolumn{2}{c}{UAT} & \multicolumn{2}{c}{AMAP}  \\
\cmidrule{2-11}          & ACC   & NMI   & ACC   & NMI   & ACC   & NMI   & {ACC} & NMI   & ACC   & NMI   &  \\
    \midrule
    DFCN  & 36.33  & 19.36 & 69.50  & 43.9  &-&-& 33.61  & 26.49  & 76.88 & \textbf{69.21}   \\
    SSGC  & 69.60  & 54.71 & 69.11  & 42.87 &-&- & 36.74  & 8.04 & {60.23} & 60.37  \\
    MVGRL & 70.47  & 55.57 & 68.66  & 43.66 &-&-& 44.16  & 21.53 & {45.19} & 36.89  \\
    SDCN  & 60.24  & 50.04 & 65.96  & 38.71 &65.78&29.47& 52.25  & 21.61 & {53.44} & 44.85   \\
    AGE   & 73.50  & \textbf{57.58} & 70.39  & \textbf{44.92} &-&-& 52.37  & 23.64 & {75.98} & -      \\
    MCGC  & 42.85  & 24.11 & 64.76  & 39.11 &66.95&32.45& 41.93  & 16.64 & {71.64} & 61.54   \\
    FGC   & 72.90  & 56.12 & 69.01  & 44.02&70.01&31.56 & 53.03  & 27.06 & {71.04} & 61.02      \\
    SCGC & 73.88 & 56.10 &71.02 &\textcolor[rgb]{ 1,  0,  0}{\textbf{45.25}} &67.73&28.65&\textbf{56.58} &28.07 &\textbf{77.48} &67.67   \\
    CCGC & 73.88 & 56.45 & 69.84 & 44.33 &68.06&30.92&56.34 &\textbf{28.15} &77.25 & 67.44   \\
    CGC & \textbf{75.15} & 56.90 &69.31 &43.61 &67.43&33.07&49.58&17.49 &73.02&63.26 \\
    DGCN & 72.19 & 56.04 & \textbf{71.27} & 44.13 &70.13&30.17& 52.27 & 23.54 & 76.07 & 66.13   \\
    DMGNC & 73.12 & 54.80 & \textbf{71.27} & 44.40 &\textbf{70.46}&\textbf{34.21}&- &- &- &- \\
    DyFSS &72.19 &55.49 &70.18 &44.80 &68.05 &26.87 &51.43 &25.52 &76.86 &67.78\\
    AMLP & \textcolor[rgb]{ 1,  0,  0}{\textbf{77.18}} & \textcolor[rgb]{ 1,  0,  0}{\textbf{57.61}} & \textcolor[rgb]{ 1,  0,  0}{\textbf{71.60}} & 44.48 &\textcolor[rgb]{ 1,  0,  0}{\textbf{74.64}} &\textcolor[rgb]{ 1,  0,  0}{\textbf{35.37}} & \textcolor[rgb]{ 1,  0,  0}{\textbf{58.40}} &\textcolor[rgb]{ 1,  0,  0}{\textbf{28.79}} & \textcolor[rgb]{ 1,  0,  0}{\textbf{79.80}} & \textcolor[rgb]{ 1,  0,  0}{\textbf{69.77}}  \\
    \bottomrule
    \end{tabular}}%
	\end{table*}%

\subsection{Experimental Setting}
For fairness, the experimental configurations for every dataset follow DGCN \cite{DGCN}, which finds the optimal solution using a grid search. The Adam optimizer is used to train the model for 200 epochs until convergence. We apply a single MLP with 100 or 500 dimensions. The learning rate is searched in \{1e-2,1e-3,1e-4\}. The hyper-parameter $k$ is set to 20 on Ogbn-arXiv and searched in \{1,2,3,5,7,8,9,10\} on other datasets. The trade-off parameter $\lambda$ is searched in \{1,1e-1,1e-2,1e-3\}. Performance is evaluated by two widely used clustering metrics: ACCuracy (ACC) and Normalized Mutual Information (NMI).

\subsection{Clustering Results}

Table \ref{Reheter} presents the clustering results in heterophilic graphs, showing that AMLP achieves the best performance in all cases, despite utilizing only a single-layer MLP. Notably, AMLP significantly outperforms recent methods addressing heterophily, such as SELENE, CGC, DGCN, and RSGL. Compared to DGCN and RGSL, the 2nd and 3rd methods, AMLP outperforms them by margins of up to 5.22\% and 16.95\% on the Washington dataset, and 4.69\% and 10.38\% on the Chameleon dataset, in terms of ACC and NMI, respectively. These results highlight AMLP's strong performance in heterophilic settings, due to its ability to learn an aggregation-aware MLP. For other methods, AMLP improves ACC and NMI by at least 7\% and 19\% on Cornell, and 14\% and 22\% on Texas, validating the importance of an unsupervised aggregation paradigm in graph clustering.

Table \ref{Rehomo} presents clustering results on homophilic graphs, showing that AMLP achieves state-of-the-art performance. While mainstream contrastive learning methods perform well, they rely on complex, rigid data augmentation designs that require domain expertise. In contrast, our simple method exceeds them in 9 out of 10 cases and achieves close results in the remaining one. Thus, our MLP can learn effective aggregation information on different graphs.

For shallow methods like MCGC, FGC, CGC, and RGSL, which apply a fixed filter to the features, AMLP outperforms them in most cases. Similarly, recent autoencoder-based methods such as DMGNC and DyFSS are also outperformed by our simple design in all cases. This highlights the importance of making representations aggregation-adaptive for graph clustering. While AMLP's NMI on Citeseer is less than 1\% behind the best-performing method, this discrepancy may be attributed to Citeseer's relatively small average number of neighbors, which limits the amount of information that can be aggregated. 


Fig. \ref{exlarge} shows the results in Ogbn-arXiv. Most baselines from Tables \ref{Reheter} and \ref{Rehomo} cannot run on this dataset due to their high complexity, so we compare AMLP with scalable methods. AMLP demonstrates dominant performance thanks to its aggregation-adaptive MLP. In terms of ACC, both AMLP and Dink-net demonstrate significant improvements over other methods (exceeding 7\%), with results that are closely comparable between the two. However, for NMI, AMLP outperforms Dink-net by a substantial margin, highlighting the effectiveness of our simple method in scaling to large datasets.

In summary, even when homophily is unknown, the aggregation-aware MLP makes node representations more discriminative, highlighting the potential of this unsupervised aggregation learning paradigm.

 \subsection{Parameter Analysis}

Our model has two parameters: the filter order $k$ and the trade-off parameter $\lambda$. To observe their impact, we first vary $k$ from 1 to 10 in Cora and Texas. Fig. \ref{parak} illustrates that both datasets perform best with a small $k$, with only a gradual decline as $k$ increases. This stability highlights our model's ability to balance information aggregation and noise suppression, adapting effectively to diverse graph structures.


To assess the impact of the trade-off parameter $\lambda$, we set its values to \{1e-4, 1e-3, 1e-2, 1e-1, 1, 10, 100\}. Fig. \ref{lams} illustrates how these values affect the model's performance. It can be observed that our model performs well across a wide range of scales, though extremely large or small values tend to result in poor outcomes. Additionally, both Cora and Texas datasets perform best when $\lambda$ is set to 0.1, suggesting that a balance between preserving aggregation information and reducing the focus on neighbor reconstruction yields better results. Specifically, setting $\lambda$ to 0.1 likely ensures a balanced emphasis on both the primary and secondary objectives, contributing to optimal performance.


\subsection{Ablation Study}
To evaluate the effectiveness of graph reconstruction, we replace the reconstructed graph with the original in the configuration labeled ``AMLP w $A$''. The results, shown in Table \ref{abla1}, indicate performance degradation in most cases. This suggests that our reconstructed graph, which incorporates a high-order grouping effect, is of higher quality than the original and improves downstream tasks. However, for Cora, the original graph $A$ outperforms the reconstructed one in NMI, indicating that Cora's original structure is already of high quality. But such perfect structures are rare in real-world data. Thus, our proposed reconstructing method can be generalized without labeled data.

We also assess the effectiveness of our loss sensitive to aggregation by removing $L_{agg}$ and only reconstructing $\widetilde{A}$ in the loss, denoted as ``AMLP w/o $L_{agg}$''. As seen in Table \ref{abla1}, this modification results in a significant performance drop in all cases, demonstrating that aggregation-aware loss is crucial. Therefore, our proposed method can make MLP adaptive to rich aggregation information in a simple yet effective way.

We conduct additional experiments to evaluate AMLP's efficiency in terms of time and GPU memory consumption. Using the Cora dataset, we compare AMLP with state-of-the-art methods, ensuring fairness by training all approaches for 200 epochs. As shown in Table \ref{time}, AMLP demonstrates at least twice the speed and the lowest GPU memory requirements among all tested methods. This superior efficiency can be attributed to AMLP's streamlined design, which consists of a single MLP and leverages graph reconstruction as a separate pre-processing step.

 \begin{figure}[t]
		\centering
			\includegraphics[width=.6\linewidth]{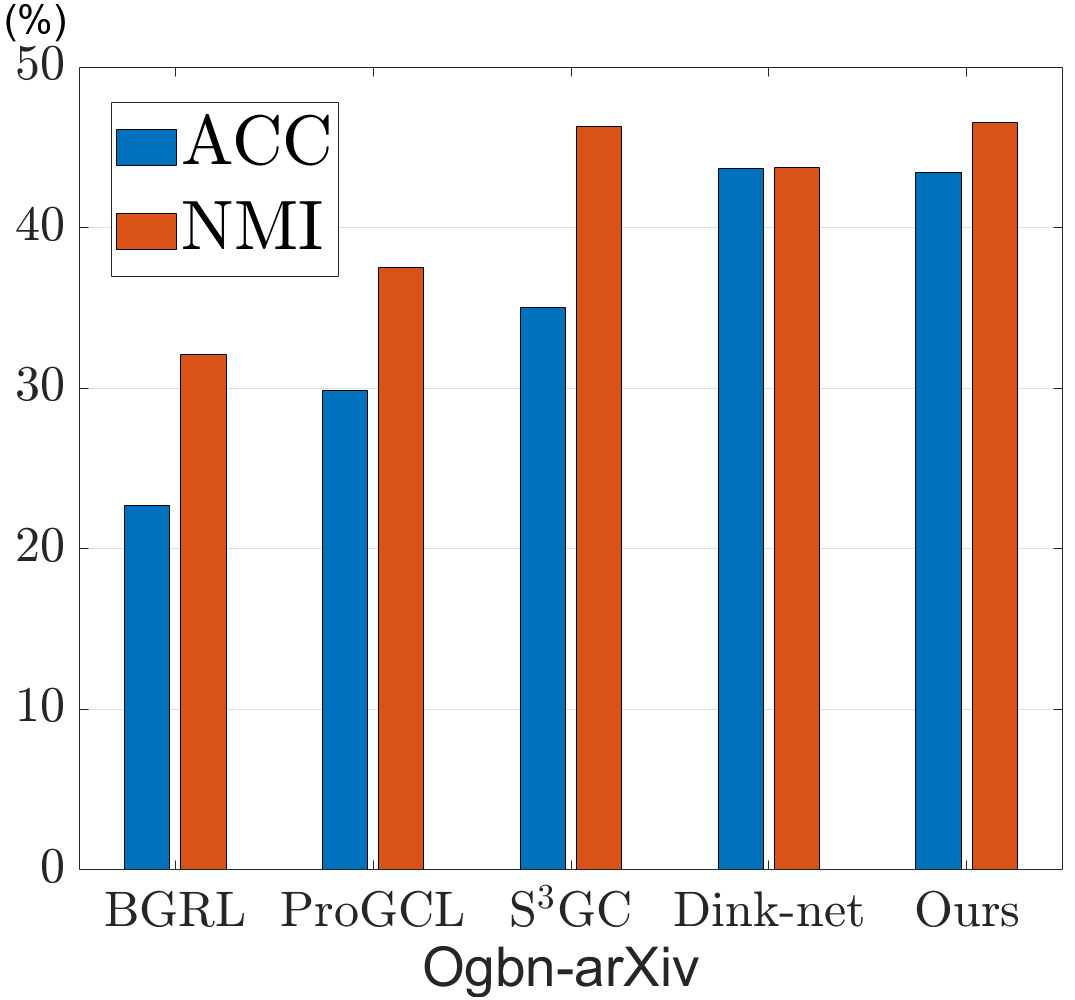}
		\caption{Results on large-scale graph.}
		\label{exlarge}
\end{figure}

\begin{figure}[t]
\centering
\includegraphics[width=0.5\textwidth]{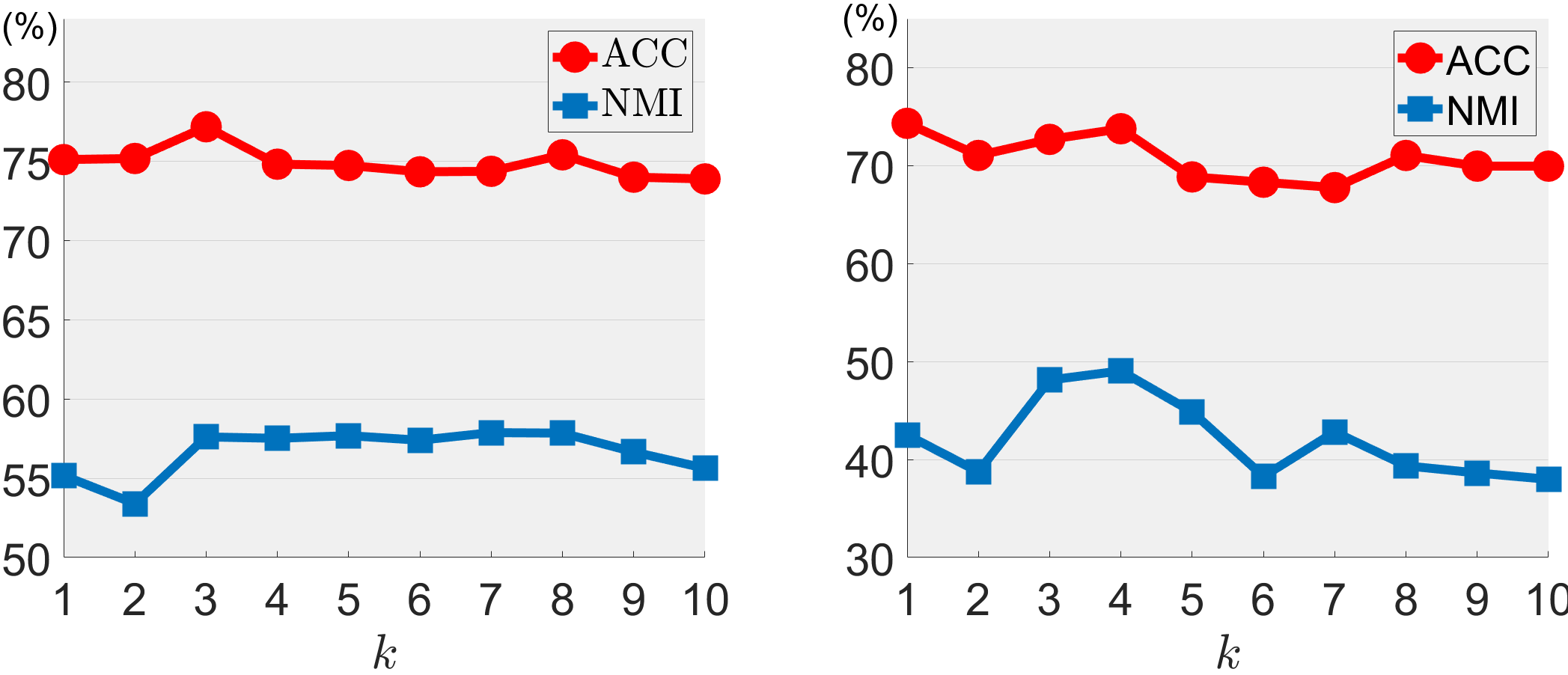} 
\caption{Sensitivity analysis of $k$ on Cora (left) and Texas (right).}
\label{parak}
\end{figure}

\begin{figure}[t]
\centering
\includegraphics[width=0.5\textwidth]{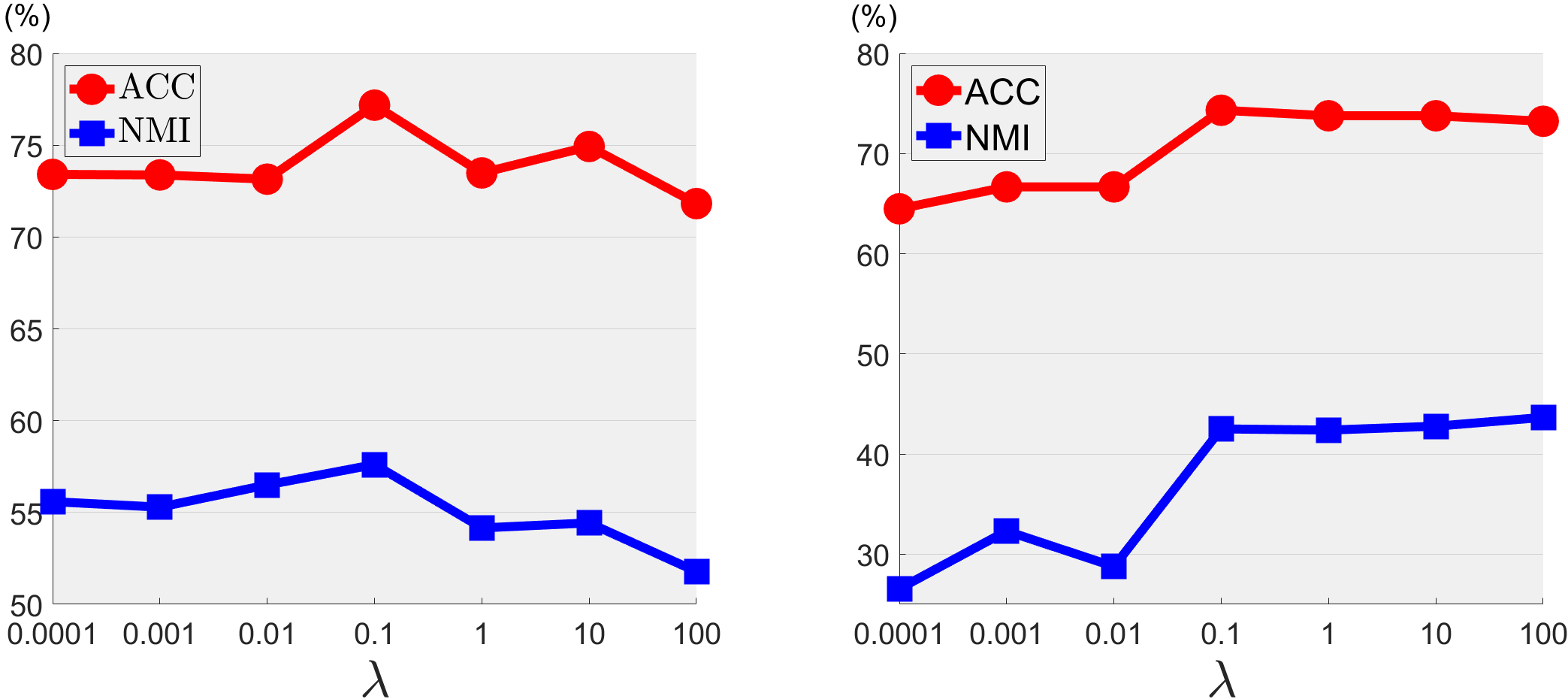} 
\caption{Sensitivity analysis of $\lambda$ on Cora (left) and Texas (right).}
\label{lams}
\end{figure}

\begin{table}[h]
\centering
\caption{Results of ablation study. The best performance is marked in \textbf{bold}.}\label{abla1}
\begin{center}
\resizebox{.4\textwidth}{!}{
\begin{tabular}{l rl| rl| rl}
\midrule
Methods & \multicolumn{2}{c}{{AMLP w $A$}} & \multicolumn{2}{c}{{AMLP w/o $L_{agg}$}} & \multicolumn{2}{c}{{AMLP}} \\ 
\cmidrule(lr){2-3} \cmidrule(lr){4-5} \cmidrule(lr){6-7} 
    & {ACC} & {NMI} & {ACC} & {NMI}& {ACC} & {NMI}\\
\midrule
Cora
&75.89 &\textbf{59.43}
&71.86 &50.91
&\textbf{77.18} &57.61
\\
UAT
&57.19 &27.67
&51.43 &25.16
&\textbf{58.40} &\textbf{28.79}
 \\
Texas
&57.38 &29.54
&73.22 &41.86
&\textbf{74.32} &\textbf{42.51}
\\
Washington
&62.61 &39.40
&69.11 &35.12
&\textbf{74.35} &\textbf{46.74}
\\
\bottomrule
\end{tabular}}
\end{center}
\end{table}

\begin{table}[]
\centering
\caption{Analysis of time (seconds) and GPU memory cost (GB) on Cora. ``-'' indicates running on CPU. }
\begin{center}
\resizebox{0.4\textwidth}{!}{
\begin{tabular}{lccccc}
\hline Methods & AGE & CGC & DGCN & DyFSS & AMLP \\
\hline Time & 48.72 & 134.26 & 97.42 & 55.33 & \textbf{21.91} \\
Memory & 1.42 & - & 1.18 & 2.23 & \textbf{0.78} \\
\hline
\end{tabular}}
    \label{time}
\end{center}
\end{table}

\section{Experiments On Node Classification}
To further verify the effectiveness of our proposed AMLP, we examine its performance on the node classification task. We use the node representations $\hat{Y}$ learned through our unsupervised methods to train a linear classifier. We compare our method with aggregation-based and self-supervised methods.

\subsection{Classification Baselines}

\textbf{Aggregation-based methods:} We compare AMLP with six state-of-the-art supervised methods that adjust the aggregation mechanism, including GPRGNN \cite{GPRGNN}, GGCN \cite{GGCN}, RAW-GNN \cite{Raw-gnn}, GloGNN++\cite{Global-hete},  CAGNN \cite{CAGNN}, and PCNet \cite{PCNet}. They adjust the aggregation mechanism since a single fixed aggregation method often fails to effectively address heterophily in graph data. For instance, GPRGNN learns PageRank weights adaptively, balancing the use of node features and topological information. GGCN employs structure- and feature-based edge corrections to send signed neighbor features under specific node degree constraints. RAW-GNN captures homophily through breadth-first random walks and heterophily through depth-first searches, replacing traditional neighborhoods with path-based ones and utilizing Recurrent Neural Networks for aggregation. GloGNN++ introduces extra parameters to balance local and global neighborhoods, thereby enlarging the receptive field. CAGNN proposes a neighbor effect for each node, allowing for better aggregation of inter-class edges through adaptive mixing of discriminative and neighbor features. Lastly, PCNet proposes a two-fold filtering with both homophilic and heterophilic aggregation, which can be approximated by the Poisson-Charlier polynomials. Notably, CAGNN can be integrated into any existing message-passing GNNs, which aligns closely with our method.

\textbf{Self-supervised methods:} We compare AMLP with five state-of-the-art self-supervised learning methods, including GRACE \cite{GRACE}, MVGRL \cite{MVGRL}, BRGL \cite{BGRL}, SELENE \cite{SELENE}, and HeterGCL \cite{HeterGCL}. GRACE designs a simple node-level contrastive learning approach. MVGRL maximizes the mutual information between representations from different structural views of graphs. BGRL uses invariance regularization to perform self-supervised representation learning without negative samples. SELENE incorporates a dual-channel embedding pipeline, r-ego network sampling, anonymization, and a negative-sample-free self-supervised learning objective to improve node representations. HeterGCL introduces an adaptive neighbor contrastive loss to capture local and global structural signals.

\subsection{Experimental Settings}
The experimental settings for the aggregation-based methods follow \cite{Global-hete}. The training/validation/test splits are in a ratio of 48\%/32\%/20\% and there are 10 different random splits. The experimental settings for the self-supervised methods follow \cite{HeterGCL}. The training/validation/test splits are in a ratio of 10\%/10\%/80\%.

\subsection{Classification Results}
Table \ref{semi} presents the classification results of aggregation-based methods. Notably, AMLP outperforms supervised GNNs in most cases. While it does not achieve the highest performance on the Pubmed dataset, its results remain competitive. On the Texas dataset, AMLP demonstrates a significant advantage over the runner-up, showcasing its effectiveness. Unlike the baseline methods that rely on complex aggregation strategies driven by labels, AMLP leverages a single-layer MLP to adapt flexibly to varying graph structures. This simple design not only enhances computational efficiency but also improves adaptability, enabling AMLP to deliver more accurate representations and superior performance across diverse datasets.

Table \ref{semi2} presents the classification results for self-supervised learning methods, where AMLP consistently outperforms all baseline models. Many baseline approaches focus on designing intricate contrastive loss functions. For instance, HeterGCL employs an adaptive neighbor contrastive loss to capture structural information from local to global perspectives. In contrast, AMLP adopts a simpler approach by modifying the loss function within a generative framework. The experimental results demonstrate that AMLP achieves better scalability and generalization while maintaining competitive performance, further underscoring the effectiveness of our proposed method.

\begin{table}[t]
\centering
\caption{Node classification accuracy  (\%) with aggregation-based methods. The best performance is marked in \textbf{bold}.}
\vspace{10pt}
\label{semi}
\resizebox{0.4\textwidth}{!}{
\begin{tabular}{@{}lllll lll@{}}
\toprule
Dataset  & Citeseer &Pubmed & Texas & Cornell\\ \midrule

    GPRGNN     
    & 77.13$_{\pm\text{1.67}} $
    & 87.54$_{\pm\text{0.38}} $
    & 78.38$_{\pm\text{4.36}} $        
    & 80.27$_{\pm\text{8.11}} $ 

  \\ 
    GGCN      
    &77.14$_{\pm\text{1.45}} $
    &89.15$_{\pm\text{0.37}} $
    &84.86$_{\pm\text{4.55}} $
    &85.68$_{\pm\text{6.63}} $

  \\
  RAW-GNN
  &75.38$_{\pm\text{1.68}} $
  &89.34$_{\pm\text{0.66}} $
  &85.95$_{\pm\text{4.15}} $
  &84.86$_{\pm\text{5.43}} $
  \\

      GloGNN++   
      & 77.22$_{\pm\text{1.78}} $ 
      & 89.24$_{\pm\text{0.39}} $
      & 84.05$_{\pm\text{4.90}}$       
      &85.95$_{\pm\text{5.10}}$
  \\
      CAGNN
      &76.03$_{\pm\text{1.16}} $
      &\textbf{89.74$_{\pm\text{0.55}} $}
      &85.13$_{\pm\text{5.73}}$       
      &81.35$_{\pm\text{5.47}}$
      
  \\
      PCNet
      &77.50$_{\pm\text{1.06}} $
      &89.51$_{\pm\text{0.38}} $
      &88.11$_{\pm\text{2.17}}$       
      &82.16$_{\pm\text{2.70}}$

   \\  \midrule
       AMLP 
        &\textbf{77.58$_{\pm\text{1.08}}$}
        &89.10$_{\pm\text{0.54}}$
        &\textbf{91.33$_{\pm\text{3.74}}$}
        &\textbf{86.05$_{\pm\text{5.16}}$}

\\
\bottomrule

\end{tabular}}
\end{table}

\begin{table}[t]
\centering
\caption{Node classification accuracy  (\%) with self-supervised methods. The best performance is marked in \textbf{bold}.}
\vspace{10pt}
\label{semi2}
\resizebox{0.4\textwidth}{!}{
\begin{tabular}{@{}lllll lll@{}}
\toprule
Dataset  & Citeseer &Pubmed & Texas & Cornell\\ \midrule

    GRACE     
    & 71.4$_{\pm\text{1.0}} $
    & 77.6$_{\pm\text{1.0}} $ 
    & 63.5$_{\pm\text{2.6}} $        
    & 56.4$_{\pm\text{2.1}} $ 

  \\ 
    MVGRL      
    & 72.3$_{\pm\text{0.7}} $   
    & 80.1$_{\pm\text{0.7}} $
    & 61.7$_{\pm\text{3.9}} $        

    & 56.2$_{\pm\text{2.4}} $

  \\
  BGRL
  &72.3$_{\pm\text{0.6}} $
  &84.7$_{\pm\text{0.4}} $
  &65.8$_{\pm\text{2.7}} $
  &56.7$_{\pm\text{2.1}} $
  \\

      SELENE   
      & 54.1$_{\pm\text{1.1}} $ 
      & 81.7$_{\pm\text{0.3}} $
      & 64.0$_{\pm\text{1.7}}$       
      &56.1$_{\pm\text{2.5}}$
  \\
      HeterGCL
      &73.0$_{\pm\text{0.6}} $
      &86.2$_{\pm\text{0.2}} $
      &74.7$_{\pm\text{3.6}}$       
      &75.5$_{\pm\text{2.8}}$

   \\  \midrule
       AMLP 
        &\textbf{73.5$_{\pm\text{0.6}}$}
        &\textbf{86.4$_{\pm\text{0.3}}$}
        &\textbf{77.8$_{\pm\text{1.5}}$}
        &\textbf{75.8$_{\pm\text{2.2}}$}

\\
\bottomrule

\end{tabular}}
\end{table}
\section{Conclusion}
This paper introduces a novel perspective on the learning paradigm of aggregation, positing that MLPs can be aggregation-adaptive to enhance the capacity of neural networks. This adaptability is achieved through the implementation of an aggregation-aware loss. Both theoretical analyses and empirical results affirm AMLP's capability to capture structural disparities and generate more discriminative node representations. This represents a significant advancement in unsupervised graph representation learning and paves the way for further research into message-passing mechanisms. Comprehensive experiments conducted on both homophilic and heterophilic graphs demonstrate that AMLP consistently surpasses existing baselines in clustering and classification tasks, showcasing its robustness and versatility across a wide range of graph structures.

\bibliography{example_paper}
\bibliographystyle{IEEEtran}


\end{document}